\documentclass[11pt,a4paper]{article}
\usepackage{a4wide}
\usepackage{color,subfigure,graphicx}
\usepackage{amsthm}
\title{Lanczos Approximations for the Speedup of \\Kernel Partial Least Squares Regression\footnote{to appear in Proceedings of the 12th International Conference on Artificial Intelligence and Statistics (AISTATS 09)}}

\author{
Nicole Kr\"amer \\Machine Learning Group\\
Berlin Institute of Technology\\
\texttt{nkraemer@cs.tu-berlin.de}
\and
Masashi Sugiyama \\
Department of Computer Science\\
Tokyo Institute of Technology\\
\texttt{sugi@cs.titech.ac.jp}
\and
Mikio L. Braun \\
Machine Learning Group\\
Berlin Institute of Technology\\
\texttt{mikio@cs.tu-berlin.de} }

\usepackage{algorithm,algorithmic}
\usepackage{amsmath,amssymb}
\newtheorem{conjecture}{Conjecture}
\newtheorem{coro}[conjecture]{Corollary}

\newtheorem{propo}[conjecture]{Proposition}

\newtheorem{defi}[conjecture]{Definition}
\usepackage{bm}
\newcommand{\bbeta}{\mbox{\boldmath$\beta$}}

\newcommand{\balpha}{\mbox{\boldmath$\alpha$}}

\bmdefine\E{E} \bmdefine\h{h} \bmdefine\J{J} \bmdefine\c{c}
\bmdefine\C{C} \bmdefine\U{U} \bmdefine\X{X} \bmdefine\D{D}
\bmdefine\K{K} \bmdefine\Z{Z} \bmdefine\x{x} \bmdefine\z{z}
\bmdefine\y{y} \bmdefine\Y{Y} \bmdefine\bfalpha{\alpha}
\bmdefine\bfmu{\mu} \bmdefine\M{M} \bmdefine\Q{Q} \bmdefine\P{P}
\bmdefine\w{w} \bmdefine\W{W} \bmdefine\p{p} \bmdefine\T{T}
\bmdefine\t{t} \bmdefine\r{r} \bmdefine\a{a} \bmdefine\B{B}
\bmdefine\I{I} \bmdefine\u{u} \bmdefine\p{p} \bmdefine\Sig{\Sigma}
\bmdefine\E{E} \bmdefine\F{F} \bmdefine\S{S} \bmdefine\s{s}
\bmdefine\w{w} \bmdefine\b{b} \bmdefine\W{W} \bmdefine\w{w}
\bmdefine\V{V} \bmdefine\v{v} \bmdefine\q{q} \bmdefine\R{R}
\bmdefine\A{A} \bmdefine\H{H} \bmdefine\d{d} \bmdefine\g{g}
\bmdefine\e{e} \bmdefine\L{L} \bmdefine\k{k} \bmdefine\N{N}
\bmdefine\y{y} \bmdefine\G{G}
\newcommand{\PP}{\mathcal{P}}

\newcommand{\betahat}{\widehat\bbeta}
\newcommand{\alphahat}{\widehat\balpha}

\DeclareMathOperator*{\sinc}{sinc}

\newcommand{\bigO}{\mathcal{O}}
\newcommand{\dof}{\operatorname{DoF}}
\newcommand{\trace}{\operatorname{trace}}

\begin{document}

\maketitle

\begin{abstract}
The runtime for Kernel Partial Least Squares (KPLS) to compute the fit is
  quadratic in the number of examples. However, the necessity of obtaining sensitivity measures as degrees of freedom for model selection or confidence intervals for more detailed analysis requires cubic runtime, and thus constitutes  a computational bottleneck in real-world data analysis. We propose a novel algorithm
  for KPLS which not only computes (a) the fit, but also (b) its approximate degrees
  of freedom and (c) error bars in quadratic runtime. The
  algorithm exploits a close connection between Kernel PLS and the Lanczos
  algorithm for approximating the eigenvalues of symmetric matrices,
  and uses this approximation to compute the trace of powers of the
  kernel matrix in quadratic runtime.
\end{abstract}

\section{INTRODUCTION}
\label{sec:intro}
Partial Least Squares (PLS) \cite{Wol1975,WolRuhWolDun1984} is a
supervised dimensionality reduction technique. Given $n$ observations $(\x_i,y_i) \in \mathbb{R}^d  \times \mathbb{R}$, it iteratively
constructs an orthogonal set $\T=\left(\t_1,\ldots,\t_m \right)\in \mathbb{R}^{n \times m}$ of $m$ latent features which have maximal
covariance with the target variable $\y=(y_1,\ldots,y_n)$. For regression, these latent components are used as predictors in a least squares fit instead of the original data leading to fitted values
\begin{eqnarray}
\label{eq:fit}
\widehat \y_m&=& \T \left(\T ^\top \T\right)^{-1} \T ^\top \y=\mathcal{P}_{\T} \y\,,
\end{eqnarray}
where $\mathcal{P}$ denotes the orthogonal projection operator.  PLS is the standard tool e.g. in chemometrics \cite{MarNae1989}, and has been successfully applied in various other scientific fields such
as chemoinformatics, physiology or bioinformatics \cite{SaiKraTsu2008,RosTreMat2003,BouStr2007}. In combination with the kernel trick \cite{RaeLinGelWol1994,RosTre2001}, Kernel Partial Least Squares (KPLS) performs dimensionality reduction and regression in a non-linear fashion.
KPLS has some appealing properties  over existing kernel methods.
Due to its iterative nature, it only relies on matrix-vector multiplications. Hence its runtime is quadratic in the number of training examples , as
opposed to --  for example -- Kernel Ridge Regression, which requires the
inversion of a large symmetric matrix, having time complexity
$\bigO(n^3)$. Furthermore, it is possible to compute the
derivative of the KPLS solution with respect to $\y$ by differentiating the iterative formulation
itself. Taking the trace of the derivative of the fitted values, one
obtains an estimate of the degrees of freedom for KPLS, which can
be used, for example, for effective model selection based on information criteria
like AIC, BIC, or gMDL \cite{KraBra2007}. The first order Taylor approximation can also be used to construct confidence intervals for PLS \cite{Den1997,PhaRilPen2002}. However, since we take the
derivative of a vector (\ref{eq:fit}), the derivative is a matrix, and the
computation of the derivative involves a number of matrix-matrix
multiplications which have time complexity $\bigO(n^3)$ for all
practical considerations.

In this work, we propose an algorithm which computes  (a) the fit of
KPLS as well as (b) its approximate degrees of freedom  and (c) confidence intervals for the KPLS solutions, all in quadratic
runtime.   These results are based on the fact that PLS is equivalent
to the Lanczos method for approximating the eigenvalues of
the kernel matrix $\K$ by the eigenvalues of a tridiagonal $m \times m$ matrix $\D$. The main contribution is to compute these approximate
eigenvalues using KPLS {\emph{itself}}. Then, using a different
formulation of the derivative of the fit in KPLS, one can
approximate the trace of powers $\K^j$ of the kernel matrix using the matrix $\D$. Since $\D$ is typically small (as it scales with the number of components), the runtime for computing the eigenvalues is cubic in $m$, and therefore, unproblematic. Since the powers of the Kernel matrices $\K^j$ are the only matrix-matrix
multiplications of order $n$ in the formula for the degrees of freedom, the approximation leads to quadratic runtime. Hence, we use the KPLS fit to approximate its degrees of freedom. In addition, using the alternative formulation of the derivative, one
can perform a sensitivity analysis of KPLS resulting in confidence intervals
on the estimates, also in quadratic runtime.

This paper is structured as follows. In Section \ref{sec:pls}, we  review  the connection between KPLS
 and Lanczos approximations, and summarize the state-of-the-art for
computing the derivative of Kernel PLS. In Sections \ref{sec:DoF} and \ref{sec:error},  we propose our
novel formulation of the derivative together with the quadratic
runtime algorithms for the degrees of freedom and the confidence intervals. We conclude with some practical examples.

PLS is closely related to Krylov methods. Therefore, we briefly recall
the definition of Krylov subspaces.  For a matrix $\C \in
\mathbb{R}^{c\times c}$ and $\c \in \mathbb{R}^c$, we call the set
of vectors $ \c,\C\c,\ldots,\C^{m-1} \c$ the Krylov sequence of
length $m$. The space spanned by these vectors is called a Krylov
space and is denoted by $\mathcal{K}_m \left(\C,\c\right)$.

\section{BACKGROUND: PLS, LANCZOS METHODS, AND SENSITIVITY ANALYSIS}
\label{sec:pls}
 In this paper, we focus on the NIPALS algorithm \cite{Wol1975} for PLS. For different forms of PLS, see \cite{RosKra2006}. The $n$ centered observations $(\x_i,y_i)$ are pooled into a $n \times d$ data matrix $\X$ and a  vector $\y \in \mathbb{R}^n$. PLS constructs $m$ orthogonal  latent components $\T=\left(\t_1,\ldots,\t_m\right) \in \mathbb{R}^{n \times m}$  in a greedy fashion. The first component $\t_1=\X \w_1$ fulfills
\begin{eqnarray}
\label{eq:critneu} \w_1=\operatorname{arg}\max_{\|\w\|=1}
\text{cov}(\X\w,\y)^2 =\frac{1}{\| \X^\top \y\|} \X^\top \y.
\end{eqnarray}
Subsequent components $\t_2,\t_3,\ldots$ are chosen such that they
maximize the squared covariance to $\y$  and that all components are
mutually orthogonal. Orthogonality can be ensured by  deflating the
original variables $\X$
\begin{eqnarray*}
\label{eq:deflation} \X_i &=& \X - \PP_{\t_1,\ldots,\t_{i-1}} \X \,,
\end{eqnarray*}
and then replacing $\X$ by $\X_i$  in (\ref{eq:critneu}). The matrix $\W=(\w_1,\ldots,\w_m)\in \mathbb{R}^{d\times m}$ can be shown to be orthogonal  as well (e.g. \cite{Hos1988}). Note furthermore that the latent components are usually scaled to unit norm.  Kernel PLS \cite{RaeLinGelWol1994,RosTre2001}
can be derived by  noting that $\w_i= \X^\top \r_i$ with
\begin{eqnarray}
\label{eq:res}
\r_i&=& \left(\y - \widehat \y_{i-1}\right) /\|\K^{1/2}\left(\y- \widehat \y_{i-1}\right)\|
\end{eqnarray}
denoting the normalized residuals, and by deflating the kernel matrix $\K$ instead of $\X$,
\begin{eqnarray*}
\K_i &=& \left(\I_n - \mathcal{P}_{\t_{i-1}}\right)
\K_{i-1}\left(\I_n - \mathcal{P}_{\t_{i-1}}\right)\,.
\end{eqnarray*}
In contrast to e.g. Principal Component Analysis, the latent components $\T$ depend on the response, and hence the fitted values (\ref{eq:fit}) are a nonlinear function of  $\y$.

Recall that in the nonlinear case, KPLS depends on the kernel parameters (e.g. the width of an rbf-kernel) and the optimal number $m$ of latent components. Thus, for model selection, one has to select the optimal combination on a grid of possible kernel parameters and components from $1$ to a maximal amount $m^*$ of components.
\subsection{KRYLOV METHODS AND LANCZOS APPROXIMATION}
To predict the output for a new observation, we have to
derive the regression coefficients $\widehat \bbeta_m$ (in the linear case) and kernel
coefficients $\alphahat _m$ (in the nonlinear case), which are defined via
$\widehat \y_m = \X \betahat_m= \K \alphahat _m$. This can be done by using the fact that PLS is equivalent  to the Lanczos bidiagonalization of $\X$
\cite{Lan1950}: The orthogonal matrices $\T$ and  $\W$ represent a
decomposition of $\X$ into a  bidiagonal matrix $\L \in \mathbb{R}^{m \times m}$ via
\begin{eqnarray}
\label{eq:bi} \X \W &=& \T \L
\end{eqnarray}
with $ l_{ij}=0$ for $i >j $ or $i <j-1$. This matrix is defined as  $\L= \T^\top \X \W $. This implies \cite{Man1987,Hos1988,RosTre2001}
\begin{eqnarray*}
\widehat \bbeta_m= \W \L^{-1} \T^\top \y& \text{ and } &\widehat \balpha_m= \R \L^{-1} \T^\top \y
\end{eqnarray*}
with $\R=(\r_1,\ldots,\r_m) \in \mathbb{R} ^{n \times m}$. Furthermore, it can be shown \cite{PhaHoo2002} that PLS is equivalent to the conjugate gradient (CG)
algorithm \cite{HesSti1952}. The latter is a procedure that
iteratively computes approximate solutions of the normal equation
$\A \bbeta=\b$ (with $\A = \X^\top \X$ and $\b=\X^\top \y$) by minimizing the quadratic
 function $1/2  {\bbeta}^\top {\A} \bbeta - {
\bbeta}^\top {\b}$
along directions that are ${\bf A}$-orthogonal. These search directions span the Krylov space defined by $\A$ and $\b$. The
approximate solution of CG obtained after $m$ steps is equal to the
PLS estimate $\widehat \bbeta_{m}$ with $m$
components. Moreover, the weight vectors $\W$ are an orthogonal
basis of $\mathcal{K}_m(\A,\b)$.

Krylov methods are also used to approximate eigenvalues of $\A$ by
``restricting'' $\A$ onto Krylov subspaces: In terms of the
orthogonal basis $\W$ of $\mathcal{K}_m(\A,\b)$, the map
\begin{eqnarray*}
\D&=&\mathcal{P}_{\mathcal{K}_m(\A,\b)}
\A\mathcal{P}_{\mathcal{K}_m(\A,\b)}
\end{eqnarray*}
is represented by
\begin{eqnarray}
\label{eq:D}\D&=&\W^\top \A \W\,.
\end{eqnarray}
$\D$ is shown to be tridiagonal, and the  $m$ distinct eigenvalues $\mu_1 >\mu_2 > \ldots  >  \mu_m$ of  $\D$ -- called Ritz values -- are good
approximations of the eigenvalues of $\A$ \cite{Saa1996}. One immediate consequence of the connection between PLS and Krylov spaces is the fact that the latent components span the Krylov space defined by $\K$ and $\K \y$. This implies that
\begin{eqnarray}
\label{eq:krylov}
\widehat \y_m&=& \mathcal{P}_{\mathcal{K}_m(\K,\K \y)} \y\,.
\end{eqnarray}

\subsection{SENSITIVITY ANALYSIS FOR  KPLS}
\label{subsec:sensitivity}

Sensitivity measures are crucial in at least two important scenarios. On the one hand, they are needed to select the correct model (in terms of  a suitable kernel and the number of
components) when using information criteria. On the other hand,  to assess the stability of the solution, one  needs to measure the influence of small noise in the
training points on the learned function. For example, areas with a
high sensitivity require further data points to stabilize the
solution in an ambiguous area. Furthermore, if for some regions, the prediction  does not depend on the training points at all, this
indicates that further data points are necessary.

Both of these questions --  model selection and stability analysis -- can be addressed by computing the derivatives
of the KPLS solution with respect to $\y$,  either of the fitted labels $\widehat \y_m$, or of the learned kernel coefficients $\widehat \balpha_m$. Let us consider the regression model
\begin{eqnarray}
\label{eq:reg}
Y_i&=& f(\x_i) + \varepsilon_i\,,\,\varepsilon_i \sim \mathcal{N}(0,\sigma^2)\,.
\end{eqnarray}
For a general regression method with fitted values $\widehat \y$, the degrees of freedom are defined as \cite{Ye9801,Efr2004}
\begin{eqnarray*}
\dof&=& E\left[\trace\left( \partial \widehat \y/\partial \y\right)   \right]
\end{eqnarray*}
with the expectation $E$ taken with respect to $Y_1,\ldots,Y_n$. An unbiased plug-in estimate of the degrees of freedom is therefore given by
\begin{eqnarray}
\label{eq:dof}
  \widehat {\dof} &=& \trace\left(
    \partial\widehat \y/\partial \y
  \right).
\end{eqnarray}
Degrees of freedom in combination with information criteria  can be used for model selection. As the KPLS solution depends nonlinearly on $\y$, the computation of the derivative is necessary. Kr\"amer \& Braun \cite{KraBra2007} derive an algorithm for the derivative of $\widehat \y_m$ by transforming the Lanczos decomposition (\ref{eq:bi}) into a Kernel representation and by exploiting its sparsity. The resulting iterative algorithm for (\ref{eq:dof}) is then used successfully for model selection. This method scales cubically in the number of examples.

For the construction of confidence intervals for a fitted kernel function
\begin{eqnarray*}
\widehat f(\x)&=& \sum_{i=1} ^n \widehat \alpha_{i} k(\x,\x_i)\,.
\end{eqnarray*}
one needs to study the influence of an infinitesimal perturbation in
the values of $\y$. If the kernel coefficients depended linearly on $\y$ via $\widehat \balpha= \H \y$, the distribution of the prediction $\widehat f(\x)$ at any point $\x$ would be given by
\begin{eqnarray}
\label{eq:varf}
\widehat f(\x)&\sim& \mathcal{N}\left( \k(\x) ^\top E\left[\widehat \balpha \right],\sigma^2 \k(\x) ^\top \H  \H^\top   \k(\x)\right)
\end{eqnarray}
with $\k(\x)=\left(k(\x,\x_1),\ldots,k(\x,\x_n)\right) \in \mathbb{R}^n$. However, as KPLS depends nonlinearly on $\y$, the distribution of $\widehat \balpha_m$ can only be determined approximately by using a first order Taylor expansion, i.e. one uses
\begin{eqnarray}
\label{eq:Hm}
\H_m&\approx& \left( \partial \widehat \balpha_m/\partial \y \right)\,.
\end{eqnarray}
To the best of our knowledge, confidence intervals for PLS have only been constructed in the linear setting, but the results can easily been extended to the Kernel case.
Phatak et.~al. \cite{PhaRilPen2002} use (\ref{eq:krylov}) to explicitly calculate the
derivative of the PLS coefficients $\widehat \bbeta_m$, and obtain an approximate distribution of $\widehat \bbeta_m$.  As the formula depends on matrix multiplications of order $(nm) \times (nm)$, this approach is computationally expensive. Furthermore, as the Krylov sequence $\K \y,\ldots,\K^m \y$ is nearly collinear, the formula is numerically unstable. In \cite{Den1997,SerLem2004}, an iterative formulation of PLS is used to construct the  derivative of $\widehat \bbeta_m$. Finally, we remark that the approach by \cite{KraBra2007} using the Lanczos decomposition can be extended to the derivative of the kernel coefficients.

The drawback of all of these approaches is their poor scalability. All
algorithms are cubic in the number of observations. In the following two sections we exploit that we do not need the derivative itself, but  only  the  trace of the derivative for
the degrees of freedom, and (\ref{eq:varf}) for the construction of
confidence intervals. The key advantage is that we can compute these approximation schemes in quadratic runtime.
\section{APPROXIMATE DEGREES OF FREEDOM IN QUADRATIC RUNTIME}
\label{sec:DoF}
The key ingredients for the derivation of approximate degrees of freedom are (1) the identification of those terms that are cubic in $n$, and (2) the approximation of those terms using Lanczos approximations.

First, we extend the results of  \cite{PhaRilPen2002} to the computation of the derivative of $\widehat \y_m$. We define the $m \times m$ matrix $\B$ via $b_{ij}= \left \langle \t_i,\K^j \y\right \rangle\,.$ The matrix is regular and upper triangular, as the latent components $\T$ are an orthogonal basis of the Krylov subspace $\mathcal{K}_m(\K,\K\y)$.
\begin{propo}
\label{pro:der}
Let $\c= \B^{-1} \T^\top \y$ and  $\V=(\v_1,\ldots,\v_m )= \T \B^{- \top}$.  We have
\begin{eqnarray*}
 \frac{\partial \widehat \y_m}{\partial \y}&=& \left[\c^\top \otimes \left(\I_n -
\mathcal{P}_{\T}\right)\right] \Q^\top + \left[ \V \otimes \left(\y - \widehat \y_m \right)^\top
\right] \Q^\top + \mathcal{P}_{\T}\\
&=& \sum_{j=1} ^m c_j \left(\I_n -
\T \T ^\top \right) \K^j + \sum_{j=1} ^m \v_j \left(\y - \widehat \y_m\right)^\top \K^j + \T \T^\top\,.
\end{eqnarray*}
Here, $\otimes$ is the Kronecker product and $\Q= \left(\K,\K^2,\ldots,\K^m\right) \in \mathbb{R}^{n \times nm}$.
\end{propo}
\begin{proof}
The first line follows by computing the derivative of the projection operator (\ref{eq:krylov}) and by applying a basis transformation from  the Krylov sequence $\K \y,\ldots, \K^m \y$ to the orthogonal basis $\t_1,\ldots,\t_m$. The latter ensures that the formula is numerically more stable. For the second line, we represent the Kronecker product as a sum.
\end{proof}
As a consequence, we yield a formula for the degrees of freedom of KPLS.
\begin{coro}
An unbiased estimated of the degrees of freedom of KPLS with $m$ components is given by
\begin{eqnarray*}
\label{eq:DoF}
\nonumber\widehat {\operatorname{DoF}}(m)&=& \sum_{j=1} ^m c_j
\operatorname{trace}\left[\left(\I_n - \T \T ^\top \right) \K^j\right] +
\sum_{j=1} ^m \left(\y - \widehat \y_m\right) ^\top \K^j \v_j +m\\
&=& \sum_{j=1} ^m c_j \operatorname{trace}\left(\K^j\right) + m \\
&&- \sum_{j=1} ^m  \left(\sum_{l=1} ^m \t_l ´^ \top \K^j \t_l \right)+ \left(\y - \widehat \y_m\right)^\top \sum_{j=1} ^m   \K^j \v_j
\end{eqnarray*}
\end{coro}
This representation of the DoF of KPLS reveals an interesting feature. The computation of  last line is quadratic in $n$, as it only involves matrix-vector multiplications. The first line however is cubic in $n$, as we need to compute the trace of powers of the kernel matrix $\K^j$ for $j=1,\ldots,m$.
\subsection{APPROXIMATE DEGREES OF FREEDOM VIA RITZ VALUES}
As explained above, PLS is equivalent to Lanczos approximations, and can be used to approximate the eigenvalues of $\X^\top \X$ via the tridiagonal matrix $\D$ defined in (\ref{eq:D}). Note that $\D$ has a kernel representation
\begin{eqnarray}
\label{eq:D2}
\D&=& \R ^\top \K^2 \R= \L^\top \L\,.
\end{eqnarray}
with $\R$ the matrix of normalized residuals defined in (\ref{eq:res}). The eigenvalues of $\D$ are called Ritz values and constitute approximations of the eigenvalues of $\X^\top \X$ \cite{Saa1996}. The quality of the approximation increases with the number $m$ of latent components, and as the computation of the Ritz values scales cubically only in $m$, an efficient strategy is to allow a generous amount of components for the computation of $\D$.

As the eigenvalues of $\X^\top \X$ correspond to the eigenvalues of $\K$ in the kernel setting, we can use Ritz values to derive an approximation of the trace of $\K^j$.
\begin{defi}[Approximate Degrees of Freedom] We define the approximate degrees of freedom of KPLS with $m$ components as
\begin{eqnarray*}
\widehat {\operatorname{DoF}}_{\operatorname{appr}}(m)&=& \sum_{j=1} ^m c_j \operatorname{trace}\left(\D_{m_{\max}} ^j\right) + m\\
&& -\sum_{j=1} ^m  \left( \sum_{l=1} ^m \t_l ´^ \top \K^j \t_l\right) +  \left(\y - \widehat \y_m\right) ^\top \sum_{j=1} ^m \K^j \v_j \,,
\end{eqnarray*}
where $\D_{m_{\max}}$ is the tridiagonal matrix defined in (\ref{eq:D2}) computed with $m_{\max} \geq m$ latent components.
\end{defi}
The computation of $\D$ only requires one additional $m_{\max} \times m_{\max}$ matrix multiplication $\L ^\top \L$. As the matrix $\D$ is of size $m_{\max} \times m_{\max}$, the runtime for the computation is cubic in the number of maximal components $m_{\max}$ (which is typically small), and quadratic in the number $n$ of examples.

\subsection{QUALITY OF THE APPROXIMATION}
Theoretically, the validity  of this approximation can be justified in terms of a deviation bound.
\begin{propo}[Saad \cite{Saa1996}]
\label{pro:Saad}
Denote by $\mu_1,\ldots,\mu_m$ the eigenvalues of
$\D$ and by $\lambda_1 \geq \ldots \geq  \lambda_n$ the eigenvalues of the Kernel matrix $\K$. We have
\begin{displaymath}
0 \leq \lambda_i -\mu_i \leq \left(\lambda_1-\lambda_n\right)
\left(\frac{\kappa_i \tan \theta_i}{C_{m-i}(1+2 \gamma_i)}
\right)^2
\end{displaymath}
with $\u_i$ the $i$th eigenvector of $\K$,
\begin{eqnarray*}
\theta_i&=& \operatorname{acos} \frac{\left \langle \y, \sqrt{\lambda_i} \u_i \right \rangle}{\|y\|_K}\,
\end{eqnarray*}
and
\begin{displaymath}
\kappa_i=\prod_{j=1} ^{i-1} \frac{\mu_j- \lambda_n}{\mu_j -
\lambda_i}\,\quad\gamma_i=
\frac{\lambda_{i}-\lambda_{i-1}}{\lambda_{i+1}-\lambda_n}\,.
\end{displaymath}
\end{propo}
Here, $C_l$ denotes the Chebychev polynomial of order $l$.

Note that $\theta_i$ is the angle between $\b$ and the $i$th eigenvector of $\X ^\top \X$ - computed in feature space.  This inequality
implies that the approximation for the $i$th eigenvalue is good under two different scenarios. Either $\lambda_i$ is already close to zero, so $ \mu_i \leq \lambda_i$  is close to zero as well. For large eigenvalues $\lambda_i$, the approximation is good if (a)
the eigenvalues of $\K$ decay fast, (b) the angle $\theta_i$ corresponding to the $i$th eigenvector is small, and (c) the index $i$ is not too large compared to $m$. Property (a) is a feature of rbf-kernels, which we use throughout the rest of the paper.  Condition (b) is typically fulfilled for the leading eigenvectors of $\K$  \cite{BraBuhMul2007,Braun0801}, and condition (c)  can be fulfilled by using a sufficient large amount $m_{\max}$ of components.

In practical applications, two important issues are the quality of the approximate degrees of freedom, and the quality of the model selection criteria based on these approximate degrees of freedom. In accordance with \cite{KraBra2007}, we choose generalized minimum description length (gMDL) \cite{HanYu2001} as model selection criterion.

The simulation setting follows the regression model (\ref{eq:reg}) with $f(x)=\sinc(x)$. We draw $n=100$ inputs $X_i$ uniformly from
$[-\pi, \pi]$ and set the  standard deviation to $\sigma=0.1$. We fit KPLS with three different  rbf-kernels of width $0.01,0.1,1$ and use different numbers $m_{\max}$ of maximal components. In addition, we compute the DoF, the approximate DoF, the gMDL criterion, and gMDL based on the approximate DoF.

\begin{figure}[htb]
 \centering
  \subfigure
  {\includegraphics[width=0.2\textheight]{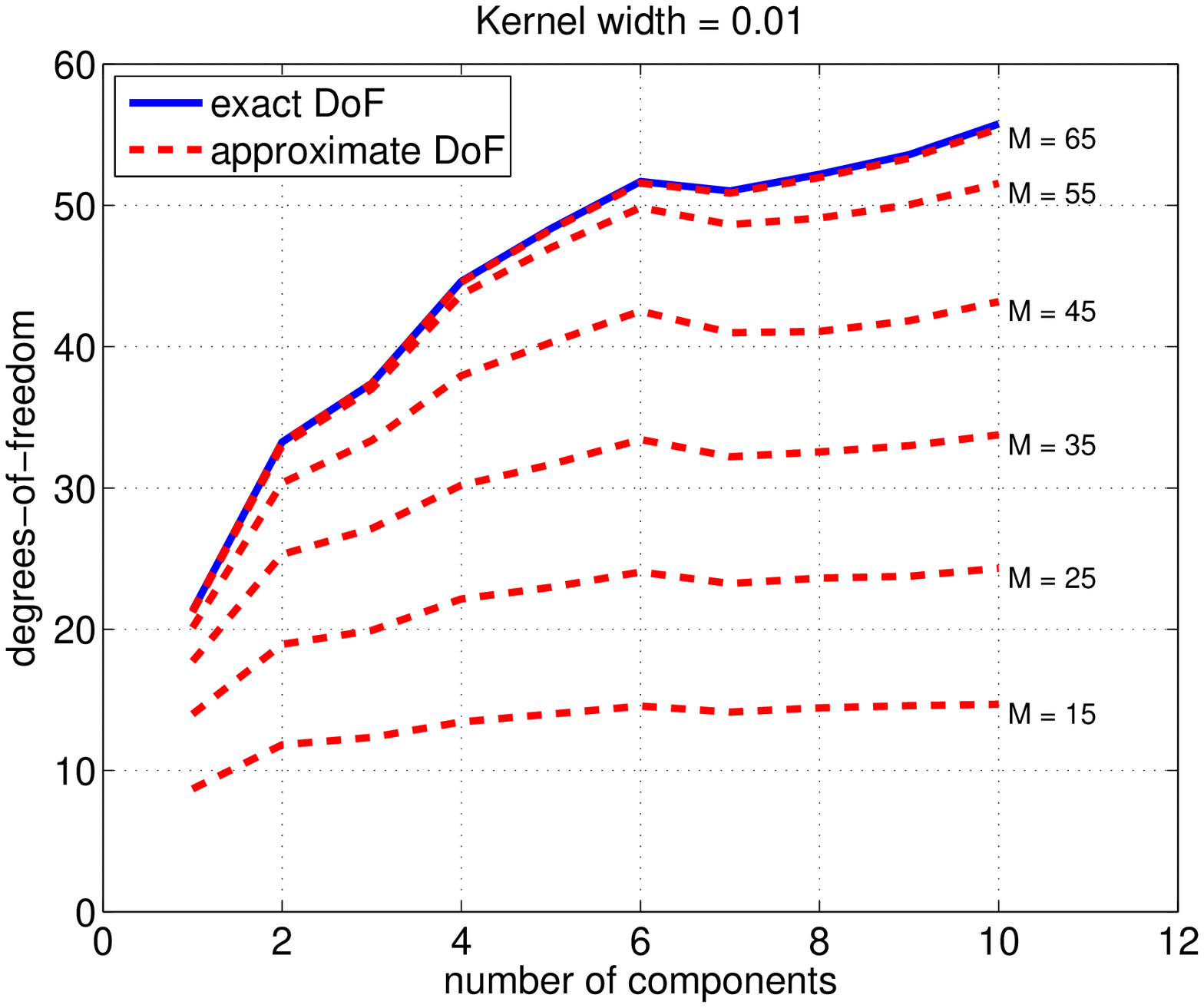}}
  \subfigure
  {\includegraphics[width=0.2\textheight]{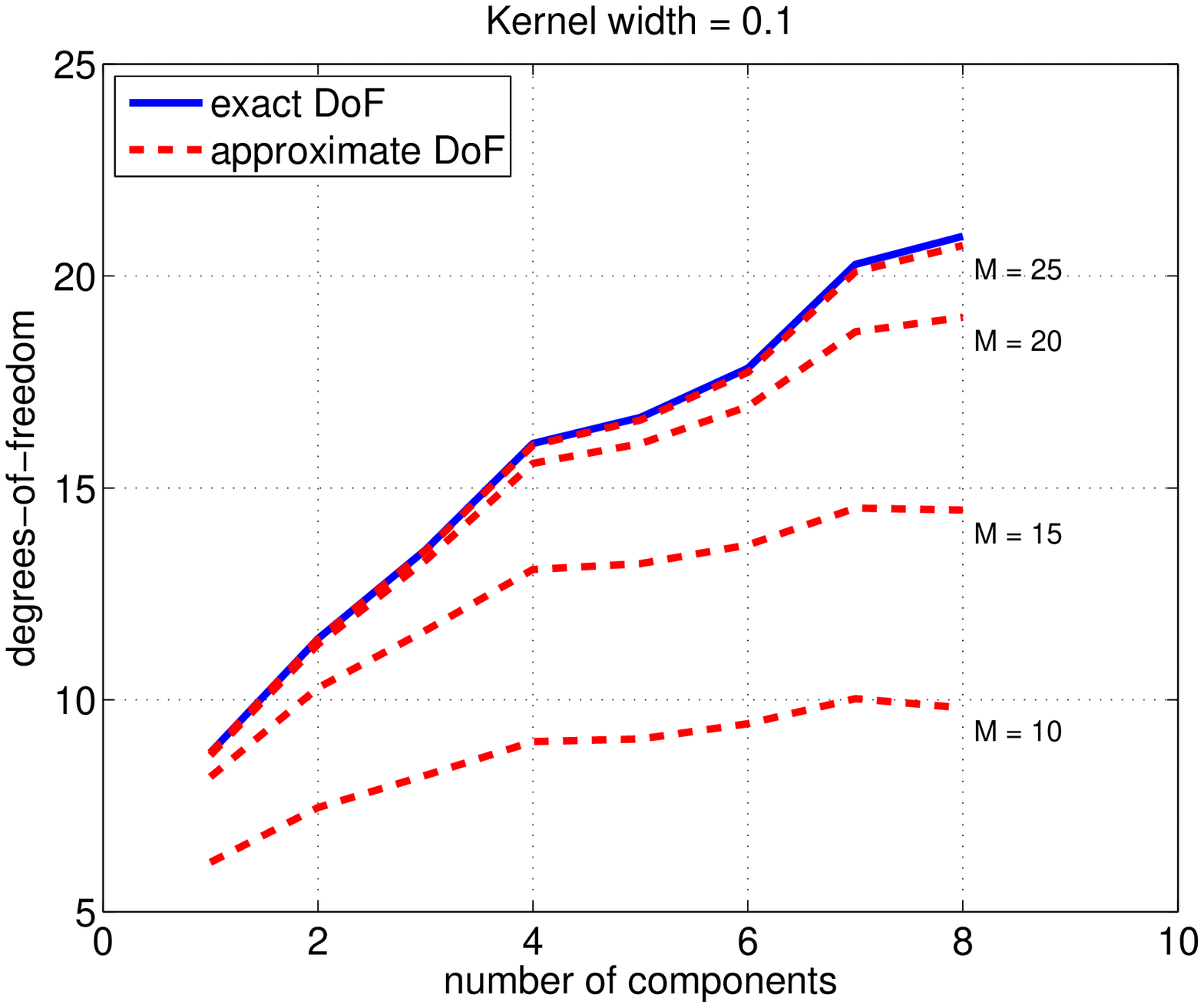}}
  \subfigure
  {\includegraphics[width=0.2\textheight]{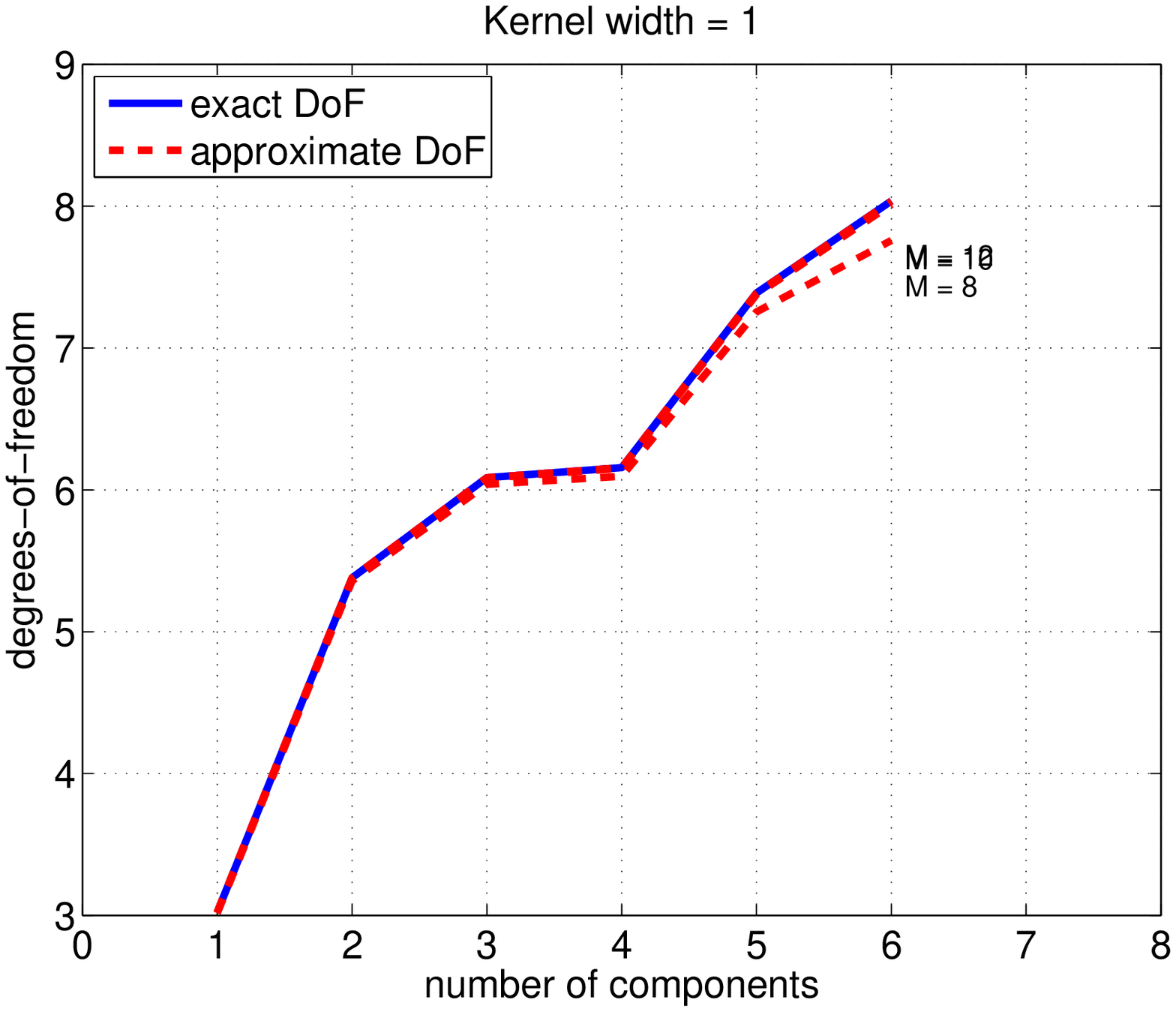}}\\
  \subfigure
  {\includegraphics[width=0.2\textheight]{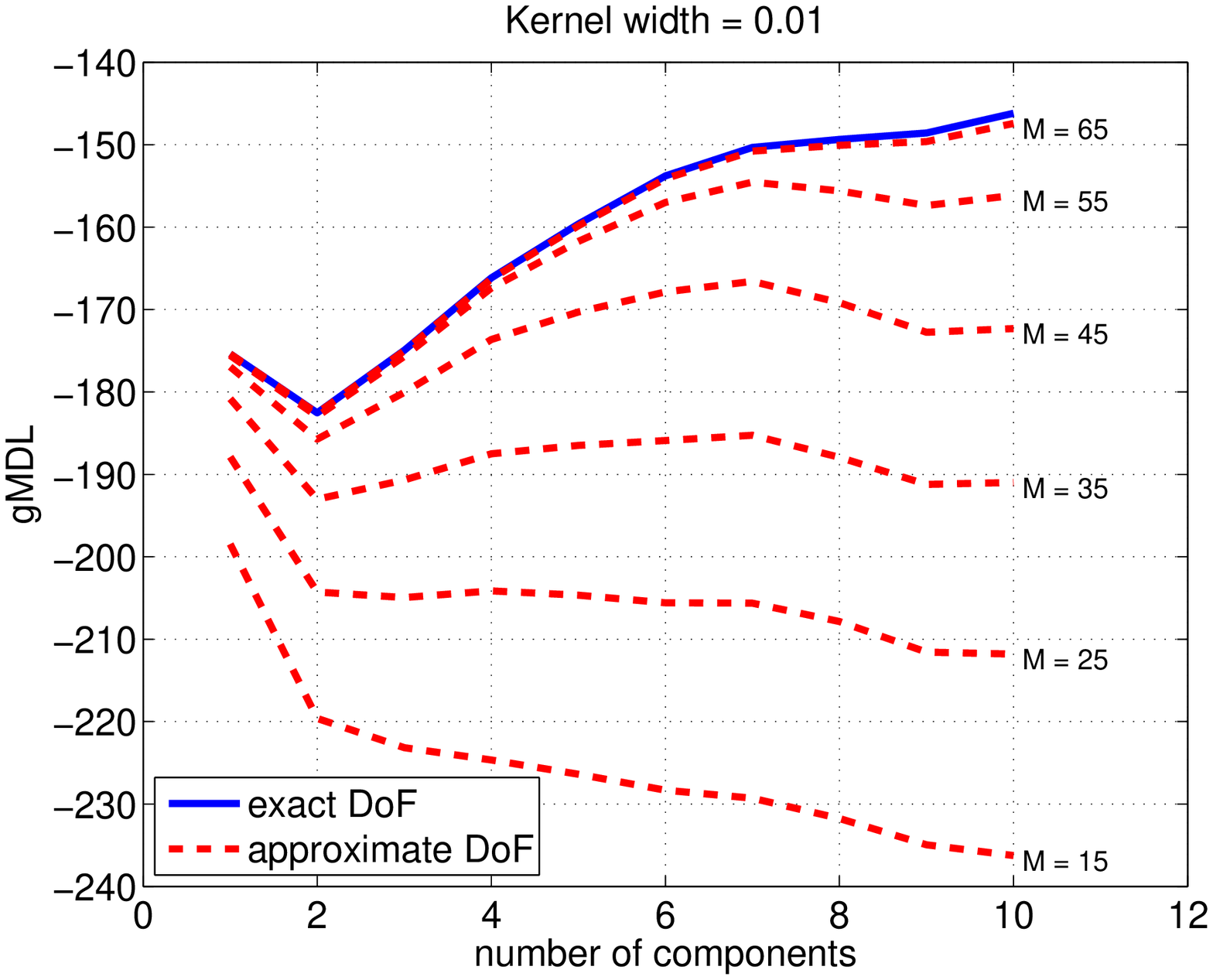}}
  \subfigure
  {\includegraphics[width=0.2\textheight]{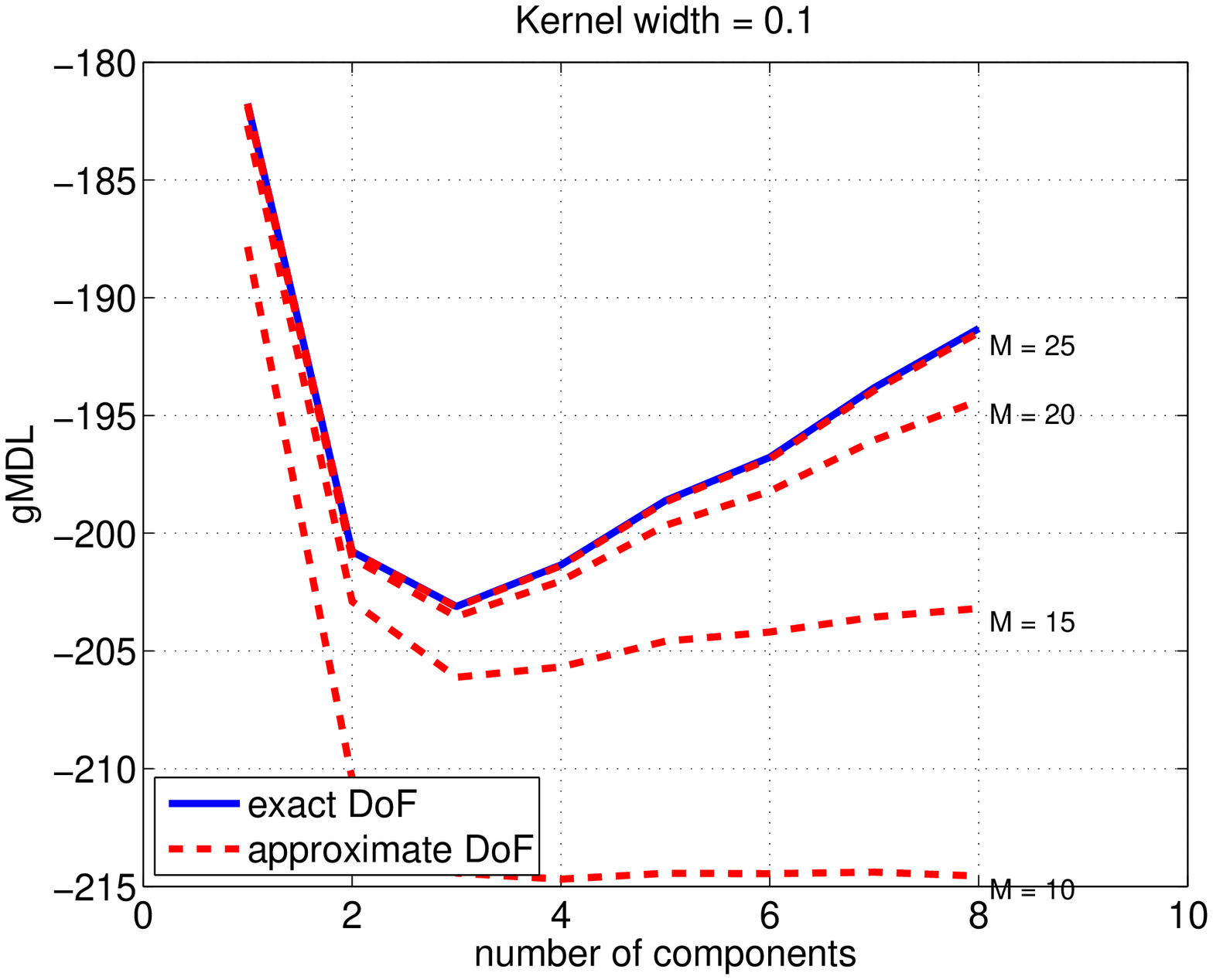}}
  \subfigure
  {\includegraphics[width=0.2\textheight]{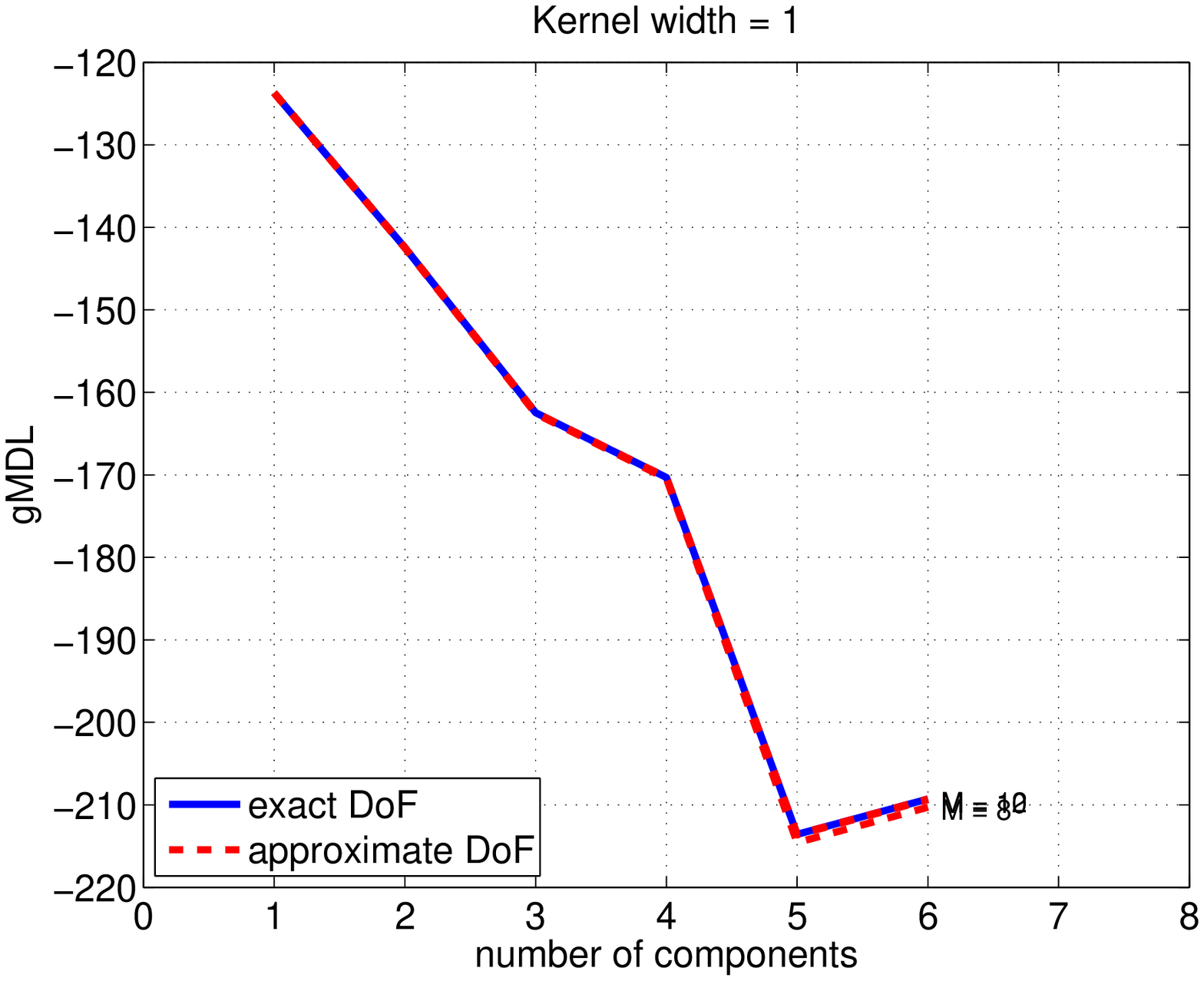}}
 \caption{Quality of the approximate degrees of freedom. Results for kernel widths $0.01$ (left), $0.1$ (center), and $1$ (right). Top row: DoF (blue line) and  approximate DoF (red dashed line) for different numbers of maximal components. Bottom row: gMDL (blue line) and  approximate gMDL (red dashed line) for different numbers of maximal components.}
 \label{fig:approx}
 \end{figure}

Figure \ref{fig:approx} displays the results for the different kernel widths $0.01$ (left), $0.1$ (center) and $1$ (right). The first row shows the degrees of freedom of KPLS (blue line) and approximate degrees of freedom of KPLS depending on the number of maximal components (red dashed line). As indicated by Proposition \ref{pro:Saad}, the approximation becomes more accurate if $m_{\max}$ is large. Furthermore, the approximation depends on the width of the rbf-kernel. For very small kernel widths (left), the eigenvalues of the kernel matrix decay very slowly, and more components are needed to compensate.
In the second line of Figure  \ref{fig:approx}, we display gMDL (blue line) and the approximate gMDL depending on  the number  of maximal components (red dashed line). The behavior of the approximation is qualitatively the same: It depends on the size of the kernel widths, and in general, it becomes more accurate if more components are used to compute $\D$.
\subsection{RUNTIME COMPARISON}
As shown above, the approximation of the Degrees of Freedom of KPLS leads to reduction in runtime from cubic to quadratic. We now illustrate that this leads to a considerable speed up even for medium sized data. We used the "kin" regression data set from the delve repository\footnote{\texttt{http://www.cs.toronto.edu/\textasciitilde delve/}}. This eight-dimensional synthetic data set is
based on a model of a robotic arm, and the task consists in predicting
the position of the arm based on the angles of its joints. It consists of 8192 data points.
\begin{figure}[ht]
\centering
 {\includegraphics[width=0.35\textheight]{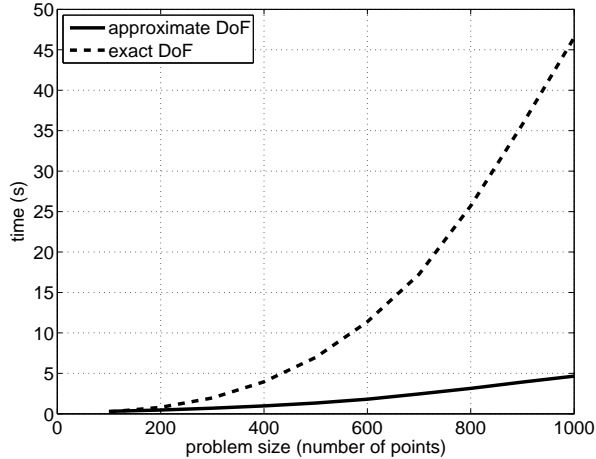}}
\caption{Comparison of runtime on the ''kin'' data set. Jagged line: KPLS with exact Degrees of Freedom for $m=10$ components. Solid line: KPLS with approximate Degrees of Freedom for $m=10$ components and $m_{max}=30$ components for the approximation of the eigenvalues of the kernel matrix. Hence, for the approximation, the effective number of  components is three times higher.}
\label{fig:runtimes}
\end{figure}
For sub-samples of size $100,200,\ldots,1000$, we compute (a) KPLS and its Degrees of Freedom for up to $m=10$ components and (b) KPLS and its approximate Degrees of Freedom for up to $m=10$ components. In both cases, we use a Gaussian Kernel. Note that for (b), we  compute $m_{max}=30$ components in order to obtain a close approximation, hence the number of KPLS iterations is three times higher for alternative (b). The runtime of both variants are displayed in Fig. \ref{fig:runtimes}. The gap between the two methods is clearly visible already for small sample sizes, and the two graphs show the expected quadratic versus cubic form. While the latter is an empirical illustration of the theoretical runtime analysis that we present above, it is important to stress that the improvement from $\mathcal{O}\left(n^3\right)$ to $\mathcal{O}\left(n^2\right)$ is not an asymptotic result but also leads to a significant improvement in runtime already for medium sized data.

\section{CONFIDENCE INTERVALS IN QUADRATIC RUNTIME}
\label{sec:error}
For the derivation of (approximate) confidence intervals (\ref{eq:varf}), we need to compute the quantity $\H_m \k(\x)$, where $\H_m \y$ is the first order Taylor approximation of the kernel coefficients $\widehat \balpha_m$. Using the representation from proposition \ref{pro:der}, we can directly
compute this matrix-vector product, even without approximating the
eigenvalues and thus compute the exact expression in quadratic runtime.

Note that Taylor expansions occur in both types of approximations, for the Degrees of Freedom as well as for the confidence intervals. However, there are essential conceptual differences. For the Degrees of Freedom, the representation in terms of derivatives (\ref{eq:dof}) is in fact no approximation but due to the assumption of normally distributed errors in (\ref{eq:reg}) which leads to Stein's Lemma \cite{Stein8101}. In this case, the Degrees of Freedom are approximated using Lanczos methods and Ritz values. In contrast, for the confidence intervals, we have to use the Taylor expansion (\ref{eq:Hm}) to obtain an approximate distribution (\ref{eq:varf}) for the KPLS parameters. The computation of the Taylor expansion  $\H_m$ defined in (\ref{eq:Hm}) is cubic in $n$ as it involves multiplications of matrices of size $n \times n$. Here, we reduce the computational cost to $\mathcal{O}\left(n^2\right)$ by cleverly exploiting the fact that the matrix-vector product $\H_m \k(\x)$ is a sufficient statistic.

\begin{propo}
We have
\begin{eqnarray*}
\H_m ^\top \k(\x)&=& \sum_{j=1} ^m \K^{j-1} \left\{c_j \left(\I_n - \K \T \N \R^\top \right)\right. \\
&&+ \left.\K  \left(\y - \widehat \y_{m}\right) \u_j ^\top \right\} \k(\x)  + \T \N \R^\top \k(\x)\,.
\end{eqnarray*}
with $\R$ denoting the matrix of normalized residuals, $\N$ denoting the $m \times m$ diagonal matrix consisting of elements $ n_{ii}= 1/\|\K \r_i\|$ and
\begin{eqnarray*}
\U&=& \left(\u_1,\ldots,\u_m\right)= \R \N \B^{- \top}\,.
\end{eqnarray*}
\end{propo}
\begin{proof}
As
\begin{eqnarray*}
\partial \widehat \y_m/\partial \y&=& \K \left(\partial \widehat \balpha_m/\partial \y\right)\,,
\end{eqnarray*}
the formula can be shown by ``canceling out'' $\K$ in the formula of the
 derivative of $\widehat y_m$, and then multiplying the formula with $\k(\x)$.
\end{proof}
\paragraph{Illustration}
Again, we use the regression model (\ref{eq:reg}), with
$f(x)=(x-1)(x+2)(x-1.5)\exp(-x^2/10)$ and $\sigma=1$. We draw $n=40$
points $X_i$ from a mixture of two normal distribution with mean $-2$ and $3$ a variance of $1$ in both cases. We fit KPLS with for two different models, (1) KPLS with
$15$ components and an rbf-kernel of width $0.1$ and (2) KPLS with $9$
components and an rbf kernel of width $1$.
\begin{figure*}
  \centering \subfigure
  {\includegraphics[width=0.3\textheight]{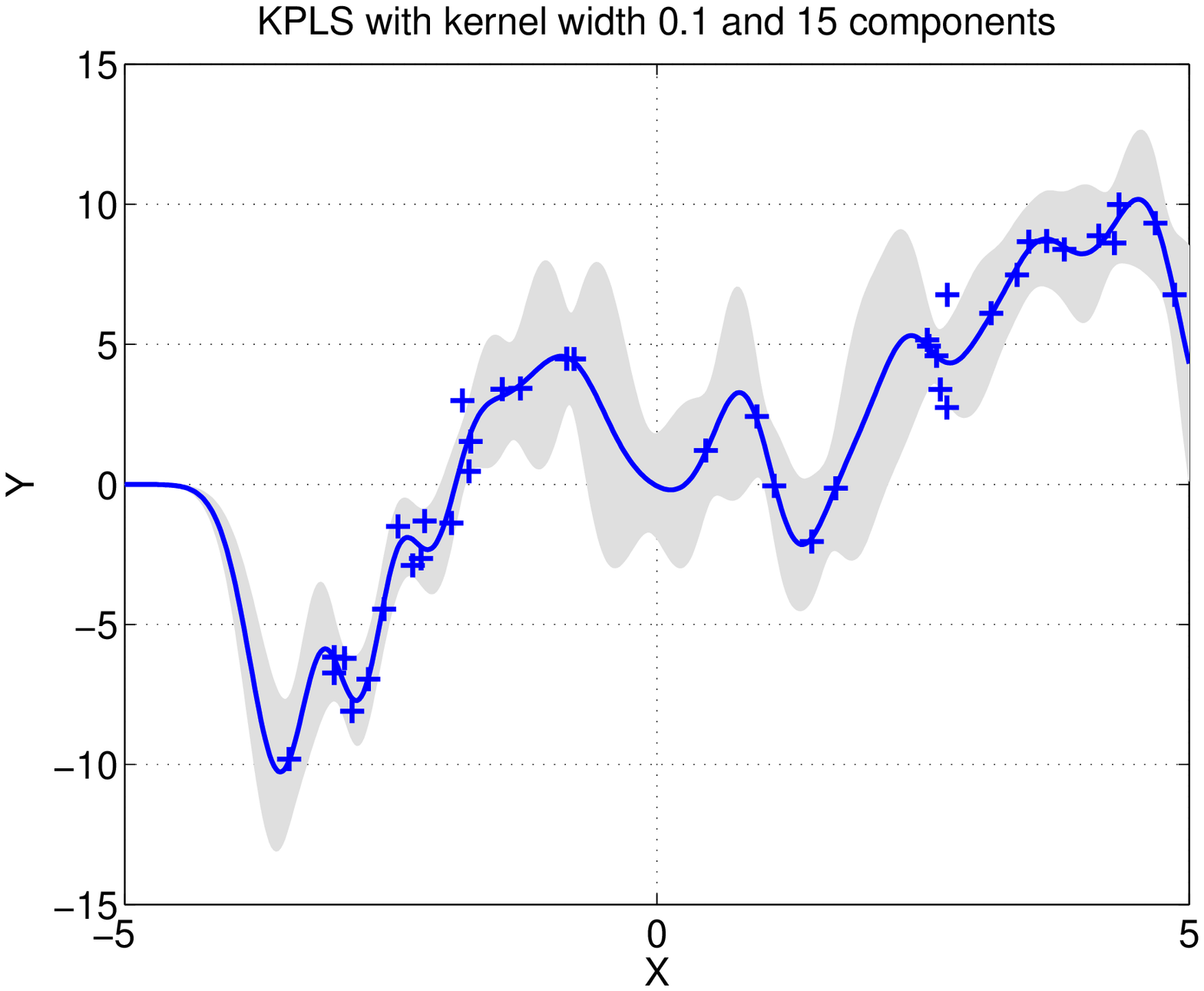}}
  \subfigure
  {\includegraphics[width=0.3\textheight]{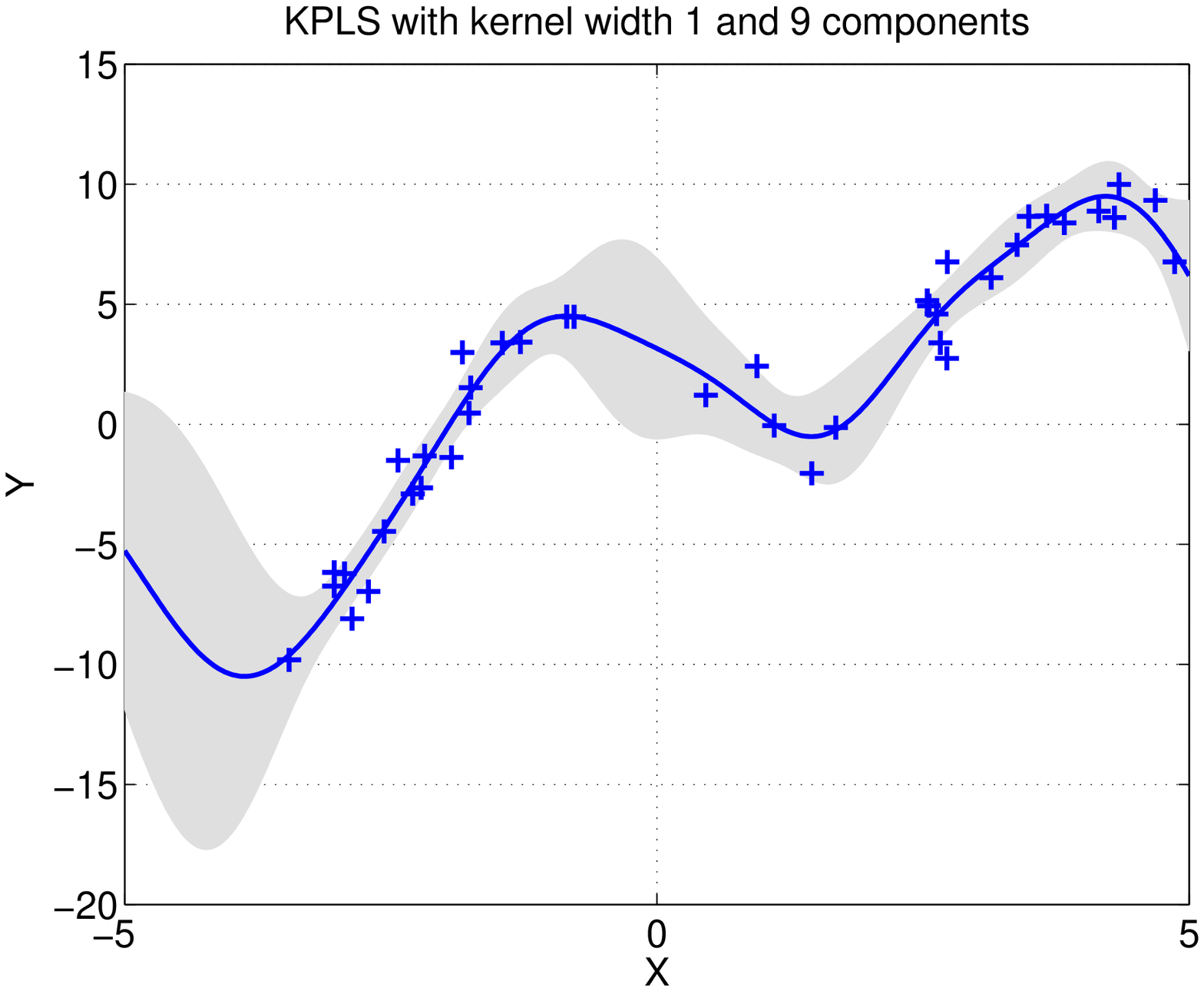}}

 \caption{Confidence Intervals for KPLS. Left: KPLS with $15$ components and an rbf-kernel of width $0.1$ Right: KPLS with $9$ components and an rbf kernel of width $1$}
 \label{fig:error}
\end{figure*}
Figure \ref{fig:error} shows the KPLS fit and its confidence intervals
(based on a level of $98$\%) for the two models.

In areas with high data density, the prediction is quite stable with
small confidence intervals. Next to such high density areas, the
predictions becomes unstable, as they can depend quite sensitively on
the neighboring data. Finally, when one moves far away from the data
points, their influence decreases to zero. This is much more apparent
in the left plot with the small kernel widths.

\section{CONCLUSION}

We proposed an implementation of the Kernel PLS method which not
only computes the fit in quadratic time, but a degree-of-freedom
estimate and confidence intervals based on a sensitivity analysis,
which formerly required cubic runtime. The latter estimates can be
used, for example, for model selection, or to measure the local
stability of the learned function. The approximation schemes exploit
the fact that Kernel PLS can be extended to compute Lanczos type
approximations of the eigenvalues as well. Together with a novel
formula for computing the derivatives of the kernel parameters
$\balpha$, these approximations allow us to replace costly computation
of powers of the kernel matrix. In summary, one obtains a Kernel PLS
algorithm which also provides relevant additional information for
model selection and provide further insight into the complexity and
stability of the learned function.

Our results capitalize on the close connection between the dimensionality reduction technique PLS on the one hand and  Krylov methods and Lanczos approximations on the other hand. While the latter two
methods are commonly used in numerical linear algebra, their
benefits for data analysis have not yet been exploited sufficiently.
Only recently (e.g. \cite{Ong0401,FreiWaMahLan2006,IdeTsu2007}) they are utilized
explicitly in a machine learning framework. Recent research results on the correspondence of penalization techniques and preconditioning of linear systems \cite{KonWhi2008,KraBouTut2008a} further underpin the strong potential of these methods.
We strongly believe
that the interplay between numerical linear
algebra and machine learning will further stimulate the field of data analysis.

\paragraph{Acknowledgement} This work is funded in part by the FP7-ICT Programme of the European Community, under the PASCAL2 Network of Excellence, ICT-216886, by the BMBF grant FKZ 01-IS07007A (ReMind), by the MEXT Grant-in-Aid for Young Scientists (A), 20680007, and by the JFE 21st Century Foundation.

\bibliographystyle{plain}

\end{document}